\documentclass[runningheads]{llncs}

\usepackage[american]{babel}
\usepackage{mathtools}
\usepackage{amssymb}
\usepackage{booktabs}
\usepackage{tikz}
\usepackage[linesnumbered,ruled,vlined]{algorithm2e}
\usepackage{nicefrac}
\usepackage{bm}
\usepackage{dsfont}
\usepackage{comment}
\usepackage{hyperref}

\newcommand{\mycircle}[2]{
    \begin{tikzpicture}[scale=#1]
        \draw[#2,fill=#2] (0,0) circle (0.5);
    \end{tikzpicture}
}

\newcommand{\expe}{{\mathbb{E}}}
\newcommand{\reals}{\mathbb{R}}

\begin{document}

\title{How Much Imprecision is Enough Imprecision in my Classifier? A Practical Elicitation Procedure}

\titlerunning{Imprecision Elicitation for Set-Valued Classifiers}

\author{Anonymous}
\author{Victor F. Lopes de Souza\inst{1,2,4} \and
Sébastien Destercke\inst{3} \and
Abdelhak Imoussaten\inst{4}}

\authorrunning{Lopes de Souza et al.}

\institute{
    EuroMov Digital Health in Motion, Univ. Montpellier, IMT Mines Alès, Montpellier, France\\
\and
    LIRMM, Université de Montpellier, CNRS, Montpellier, France
\and
   Heudiasyc, Université de Technologie de Compiègne, CNRS, Compiègne, France
\and
   SyCoIA, IMT Mines Alès, Alès, France
}

\maketitle

\begin{abstract}
Set-valued classifiers, whether derived from precise probabilities and an adapted cost function, from convex sets with a robust inference mechanism, or from conformal methods, are routine options to obtain more robust, trustworthy predictions. However, there is a lack of operational tools to measure how robust or imprecise a given user is ready to be when receiving predictions, that is how much precision he/she is ready to let go in exchange of more accuracy. This is why we propose, in this paper, practical and operational elicitation procedures to measure the user proneness to set-valued predictions. The effectiveness of the iterative elicitation procedure in converging to the target parameter value is demonstrated on both tabular and image datasets drawn from standard machine learning benchmarks. The results show that the procedure also presents the user with a small number of instances, highlighting the practicality of the approach for real-world applications aimed at identifying the decision maker's optimal behavior when faced with imprecision. 

\vspace{1em}

{Code available at \url{https://anonymous.4open.science/r/elicitation_imprecision_utilities-11A2}}

\keywords{Set-valued classification \and Imprecision elicitation \and Robust prediction \and Decision making}
\end{abstract}

\section{Introduction}

Quantifying uncertainty in predictive classification models has always been an important issue, but it has become a major challenge for ML models in recent years. This has allowed the development of many different lines of research, among which are uncertainty measures and their disentanglement (mostly between so-called epistemic and aleatoric uncertainties) \cite{hullermeier2021aleatoric}, evidential \cite{yuan2020evidential} and Bayesian deep learning approaches \cite{mackay1992practical}, or the production of set-valued predictions. 

This latter trend has given rise to many methodological alternatives, such as using conformal approaches \cite{vovk2005algorithmic}, deriving set predictions from probabilities and utilities on sets \cite{del2009learning} \cite{imoussaten2022cautious}, deriving set predictions from probability (a.k.a. credal) sets \cite{zaffalon2001statistical}, or skeptical inferences consisting in predicting all potentially optimal predictions \cite{nguyen2018reliable}. 

Together with the problem of learning such set-valued classifiers comes the need to evaluate set-valued predictions, and therefore to define well-founded valuation functions. Such functions indeed fill two essential purposes: in the learning phase, they allow to clearly define what is an optimal classifier, and in the production/inference phase, they allow to send the most promising subset to the end-user. However, allowing for set-valued predictions means that we must achieve a compromise between being accurate (containing the true class) and being informative (having a set as small as possible given our information). 

A way to do that is to define a utility function over set-valued prediction, extending classical accuracy or more complex cost functions, that typically relies on a parametric form where the parameter is used to settle the trade-off between accuracy and informativeness  \cite{del2009learning} 
\cite{imoussaten2023study} \cite{zaffalon2012evaluating}. Another connected trend is the development of proper scoring rules for models outputting as predictions convex sets of probabilities, with a possible application of such scores being to learn optimal credal models \cite{hofman2024quantifying}. Note however that this differs from considering utilities for set-valued predictions, as those can indeed come from credal sets, but can also be produced by other approaches. 

As interesting as those research may be, they do not address the question of how to choose the right value for the mentioned trade-off between accuracy and informativeness.  To our knowledge, we have nowadays no operational means to identify this right value. As this is to some extent a subjective assessment that depends both on the end-user attitude towards imprecision and on the applicative context, it seems natural to seek an operational procedure where we elicit this value from the end-user. 
This is the problem we consider in this paper, that is structured as follows: Section~\ref{sec:setvalclass} recalls the necessary basics of set-valued classification, and Section~\ref{sec:elicitation} exposes the bulk of our proposal, which consists in two elicitation procedures to obtain the user behaviour towards set-valued predictions. Section~\ref{sec:related_works} discusses related work, and Section~\ref{sec:experiments} presents some experiments.  

\section{Set-valued classifiers}
\label{sec:setvalclass}

Given some input space $\mathcal{X}$ and some finite output space $\mathcal{Y}=\{y_1,\ldots,y_K\}$, we assume as is classically done in machine learning that data are issued from a distribution $P$ over $\mathcal{X} \times \mathcal{Y}$. We consider that a set-valued classifier $h:\mathcal{X} \to 2^\mathcal{Y}$ outputs as prediction a subset $h(x)=Y$ of $\mathcal{Y}$  when observing an instance $x$ of the input space, and is an element of a large hypothesis space $\mathcal{H}$.

We also consider that the end-user has some weak order $\succ$ over the different models of $\mathcal{H}$, where $h \succ h'$ means that the user prefers the predictive behaviour of $h$ over $h'$, but that this weak order is not known to us. Our goal in this paper is to discuss the assumptions we can make about this weak order, and how we can uncover the preferences of the user through an operational elicitation procedure.

A classical means to model such weak order is to consider that it can be numerically represented by a well-defined utility function $u:\mathcal{Y} \times 2^\mathcal{Y} \to \reals$, in which case we can use the classical expected utility criterion
\begin{align}\label{eq:compa_class}
h \succeq_{u} h' \Leftrightarrow \expe_P(u(h)) \geq   \expe_P(u(h'))\end{align}
where $\expe_P(u(h))=\int_{\mathcal{X} \times \mathcal{Y}} u(h(x),y) d P(x,y)$. That is, a model $h$ is preferred to $h'$ whenever its expected utility is higher. 
The problem mentioned in the last paragraph then amounts to uncovering the exact form of the utility function, which in the case of set-valued prediction is often parametric. This is the approach we will follow in Section~\ref{sec:elicitation}.

While the above framework gives us a theoretical way of define our problem, it is of course impossible to obtain a weak order on a potentially infinite amount of models, and even to compute the expected values, as those would require an access to the theoretical, unobserved $P$. In practice we will most of the time have at our disposal a finite number of observed (i.i.d.) data $\mathcal{D}=\{(x_1,y_1),\ldots,(x_n,y_n)\}$, together with some trained models $h_1,\ldots,h_T$, that we have to use in order to elicit the preference structure $\succ$. Example~\ref{exm:setclassifiers} illustrates the situation we are facing. 

\begin{example}\label{exm:setclassifiers} Figure~\ref{fig:setcomparison} pictures the situation when $\mathcal{Y}=\{\mycircle{0.2}{red},\mycircle{0.2}{blue},\mycircle{0.2}{black}\}$, with some details about four classifiers. While classifier $h_1$ tends to be very informative, it also seems to commit quite a number of mistakes. At the other end of the spectrum is classifier $h_T$, that is poorly informative, yet always manages on the displayed instances to capture the true class. Classifiers $h_2$ and $h_3$ are in-between those two situations, and can be considered difficult to compare.  
\end{example}

\begin{figure}\small \centering
	\begin{tabular}{ccccc@{\hspace{0.05pt}}c@{\hspace{0.05pt}}c}
    Instance & True class & $h_1(x)$ & $h_2(x)$ & $h_3(x)$ & \ldots &  $h_T(x)$\\
	\hline
	$x_1$ & $\mycircle{0.2}{red}$  & $\{\mycircle{0.2}{red}\}$ & $\{\mycircle{0.2}{red},\mycircle{0.2}{blue}\}$  & $\{\mycircle{0.2}{red}\}$ &  & $\{\mycircle{0.2}{red},\mycircle{0.2}{black}\}$ \\	
	$x_2$ & $\mycircle{0.2}{blue}$  & $\{\mycircle{0.2}{black}\}$ & $\{\mycircle{0.2}{blue}\}$  & $\{\mycircle{0.2}{blue},\mycircle{0.2}{black}\}$ &  & $\{\mycircle{0.2}{red},\mycircle{0.2}{black},\mycircle{0.2}{blue}\}$ \\	
	$x_3$ & $\mycircle{0.2}{red}$  & $\{\mycircle{0.2}{red},\mycircle{0.2}{blue}\}$ & $\{\mycircle{0.2}{red},\mycircle{0.2}{black}\}$  & $\{\mycircle{0.2}{black}\}$ &   & $\{\mycircle{0.2}{red},\mycircle{0.2}{blue},\mycircle{0.2}{black}\}$ \\
	\vdots & \vdots & \vdots & \vdots & \vdots & & \vdots \\
	$x_n$ & $\mycircle{0.2}{black}$ &$\{\mycircle{0.2}{blue}\}$ & $\{\mycircle{0.2}{red},\mycircle{0.2}{blue}\}$  & $\{\mycircle{0.2}{blue},\mycircle{0.2}{black}\}$ &   & $\{\mycircle{0.2}{red},\mycircle{0.2}{black}\}$ \\
	\end{tabular}
\caption{set-valued classifiers comparison: illustration}
\label{fig:setcomparison}
\end{figure}

Of course, it makes no sense to ask a user to compare, say, a hundred results of classification and to form an opinion about which classifier is preferable. In practice, we want to expose the user to a small subset of instances on which the user will form a preference.  

In general, we will not assume any specific relationship between the classifiers, however in the literature it is quite common to consider \emph{nested} classifiers, where $h_i(x) \subseteq h_{i+1}(x)$. It is actually a quite common case when building set-valued predictors, as many set-valued predictions techniques rely on a hyper-parameter fixing the amount of imprecision: utility parameter for probabilistic approaches, confidence level for conformal approaches, or learning speed parameter for credal approaches. 

\section{Eliciting disposition to imprecision}
\label{sec:elicitation}

As we said, we will treat the problem of eliciting the user preference by presenting a subset of specific instances to him/her. We first consider, from a rather qualitative point of view, the instance-wise preferences a user may have for a series of peculiar situations regarding an instance $(x,y)$ and two classifiers $h,h'$. We will then deal with the numerical aspects of the elicitation procedure.

\subsection{Some instance-wise, qualitative aspects}
\label{sec:qualitative_aspects}

We will denote by $\succ^x$ a preference the user have for a specific instance $x$. When comparing $h(x)$ and $h'(x)$, we have a first set of situations where we think the preferences are cognitively clear :
\begin{itemize}
	\item Right and wrong: in the case where $y \in h(x)$ and $y \not\in h'(x)$, then $h \succ^x h'$, as a user would always prefer an accurate answer over a wrong one. This would be the case for $x_3$ in Figure~\ref{fig:setcomparison} when comparing $h_2$ and $h_3$, as $h_3$ makes an inaccurate answer. 
	\item Rights and Included: in the case where $y \in h(x)$, $y \in h'(x)$ and $h(x) \subset h'(x)$, then $h \succ^x h'$, as if two predictions are accurate, a more informative one is clearly preferable. 
\end{itemize}
We then have the situations where the preferences are less clear from a cognitive standpoint: 
\begin{itemize}
	\item Rights, not included: in the case where $y \in h(x)$, $y \in h'(x)$ and $h(x) \not\subseteq h'(x)$, $h'(x) \not\subseteq h(x)$, then there is no clear preference between the two situations, especially if $|h(x)|=|h'(x)|$. If this last inequality is not true, we could think that in general $h \succ^x h'$ if $|h(x)| < |h'(x)|$, as one may again prefer a more informative prediction. However, it is not entirely clear that this will always be the case, in particular if there are some relationships between the classes: one may prefer a more imprecise but consistent prediction to a less consistent yet more precise prediction. Think for instance of a case where the true class is $\{car\}$, and compare a classifier $h$ predicting three vehicles, among which a car, with a classifier $h'$ predicting two classes: car and an animal. It is unclear that $h'$ will necessarily be preferred by a user, as it would be more precise but semantically more heterogeneous. Finally, note that this situation cannot happen if we compare a precise with an imprecise classifier, or a sequence of nested classifiers. 
	\item Wrongs and included: in the case where $y \not\in h(x)$, $y \not\in h'(x)$ and $h(x) \subset h'(x)$, it is tempting to consider the default assumption that $h \succ^x h'$, as being wrong while providing less possible answers may be perceived as better than still being wrong while giving more possible answers. However, from an applicative perspective, the answer does not seem so obvious. Think for instance of a comparison between a precise wrong classifier and an imprecise and also wrong classifier: it may be that the end-user prefers the imprecise classifier, as more imprecision may be perceived as a signal of unreliability of the prediction. 
	\item Wrongs, not included: in the case where $y \not\in h(x)$, $y \not\in h'(x)$ and $h(x) \not\subseteq h'(x)$, $h'(x) \not\subseteq h(x)$, there is again no clear preference between the two situations, especially if $|h(x)|=|h'(x)|$. If this last inequality is not true, we could think that in general $h \succ^x h'$ if $|h(x)| < |h'(x)|$, but this seems even less clear than when predictions contains the ground truth, as the previous comment also applies. 
\end{itemize}

While these remarks may be obvious, they also indicate that we cannot present any subset of instances to the end user, and must take care of presenting a variety of situations. Take for example the nested case where $h_1$ is a precise, determinate classifier. If we were to present a subset of instances where $h_1$ is always right, then obviously we have $h_1 \succ^x h_2$ for every instance in this subset, and the reverse would happen if we selected a subset of instances where $h_2$ is imprecise but right, while $h_1$ is wrong. 

The remarks we made here are reasonable whatever the utility or loss function we use to evaluate classifiers. Let us now have a closer look at numerical procedure to elicit our robustness behaviour when considering specific utility functions. 

\subsection{Numerical elicitation}


We now come back to more numerical considerations, as well as to our initial purpose which was to uncover the end-user preferences with regard to imprecision in a given problem. We will assume that our utility function depends on a single parameter $\theta$, and for a given value of $\theta \in [\underline{\theta},\overline{\theta}]$, we will denote that corresponding utility $u_\theta(h(x),y)$. We will also assume that the end-user behaviour with respect to imprecision corresponds to a specific yet unknown value $\theta^*$. Our goal is then to identify the value $\theta^*$ that corresponds to the end user preferences regarding imprecision, by presenting the end user with a subset $\mathcal{D}' \subseteq \mathcal{D}$ and the results of two classifiers $h,h'$ such that it gives us useful feedback about $\theta^*$, typically by reducing the possible values of $\theta$ by half.

\begin{remark}
The utility function  mentioned in Equation~\eqref{eq:compa_class} used to determine the weak order between classifiers is often of the following generic form:
\begin{equation}
	\label{eq:discount_utility} u(h(x),y)=g(|h(x)|) \mathds{1}_{y \in h(x)}
\end{equation}
where $g$ is decreasing in $|h(x)|$ (the bigger the returned set, the lower the reward). Such forms are usually asked to be $\nicefrac{1}{|h(x)|}$-convex, meaning that 
$$ \frac{1}{g(|h(x)|+1)} \leq \frac{\frac{1}{g(|h(x)|)} + \frac{1}{g(|h(x)|+2)}  }{2},$$
which ensures that optimal set-valued predictions can be determined efficiently~\cite{mortier2021efficient}. In this paper, we consider a specific instance of such utility functions.
\end{remark}

Just as we denoted by $\succ^x$ a preference the user have for a specific instance $x$, we will denote by $\succ^\mathcal{D'}$ a preference between models given the predictions for some subset $\mathcal{D'} \subseteq \mathcal{D}$ of instances. Considering the empirical approximation $\expe_P(u_\theta(h)) \approx \sum_{x \in \mathcal{D}'} u_\theta(h(x),y)=\hat{\expe}_P(u_\theta(h)) $, an assessment $h \succ^{\mathcal{D'}} h'$ would then amount to state
\begin{equation}\label{eq:pref_values}
	\sum_{x \in \mathcal{D}'} u_\theta(h(x),y)\geq\sum_{x \in \mathcal{D}'} u_\theta(h'(x),y)
\end{equation}
which can then be transformed into an assessment over the potential value of $\theta^*$, as Equation~\eqref{eq:pref_values} tells us that 
{\small \begin{equation}\label{eq:pref_inequality}\theta^* \in \{\theta \in [\underline{\theta},\overline{\theta}]: \sum_{x \in \mathcal{D}'} u_\theta(h(x),y)\geq\sum_{x \in \mathcal{D}'} u_\theta(h'(x),y)\},\end{equation}}
For a given value $\delta$ of $\theta$, we will denote by $\mathcal{D}_{\delta} \subseteq \mathcal{D}$ a subset of instances such that 
\begin{equation}
	\hat{\expe}_P(u_\delta(h))= \hat{\expe}_P(u_\delta(h')).
\end{equation}
In general, the set in Equation~\eqref{eq:pref_inequality} may not be an interval of values, especially if Equation~\ref{eq:pref_values} is non-linear in $\theta$. In fact, one has the following property

\begin{proposition}\label{prop:interval_form}
	The set 
	$$\{\theta \in [\underline{\theta},\overline{\theta}]: \sum_{x \in \mathcal{D}'} u_\theta(h(x),y)\geq\sum_{x \in \mathcal{D}'} u_\theta(h'(x),y)\}$$
	is an interval iff $\sum_{x \in \mathcal{D}'} u_\theta(h(x),y) - \sum_{x \in \mathcal{D}'} u_\theta(h'(x),y)$ is quasi concave over $[\underline{\theta},\overline{\theta}]$. 
\end{proposition}

\begin{proof}
	The proof simply relies on the fact that the set $\{\theta \in [\underline{\theta},\overline{\theta}]: \zeta(\theta) \geq 0\}$ is an interval if and only if $\zeta$ is quasi-concave on the domain $[\underline{\theta},\overline{\theta}]$. Indeed, the level sets $\{\theta:\zeta(\theta)>c\}$ of $\zeta$ are intervals if and only if $\zeta$ is quasi-concave.
\end{proof}

For instance, this condition is not met if $u_\theta(h(x),y)$ is strictly convex (or concave) in $\theta$, indicating that meeting this constraint can be quite tedious.  It is of course possible, but impractical to work with non-convex regions of the parameter $\theta$. In the next section, we will look in particular at the parametric form
\begin{equation}\label{eq:param_utility}
	u_\theta(h(x),y)= \left(\theta + \frac{1-\theta}{|h(x)|}\right)\frac{1}{|h(x)|}  \mathds{1}_{y \in h(x)} ,
	\end{equation}
that is a quite common choice for set-valued prediction utility, and also has the interest that it is linear in its parameter $\theta$ \cite{zaffalon2012evaluating}, therefore meeting the condition of Proposition~\ref{prop:interval_form}. Note that the parameter $\theta$ should lie in the interval $[1,3]$: indeed, any value below $1$ would amount to selecting only precise predictions, in which case considering set-valued classifiers would make no sense, and any value above $3$ would amount to systematically choosing a prediction containing 2 classes rather than a correct single-class prediction—even when all the probability is concentrated on a single class, meaning that we will systematically be imprecise in our predictions (something that is hardly acceptable for most users).

Let us now study two different approaches to parameter elicitation: both are of incremental nature, the first assuming that we have at our disposal large samples of situations in which we can find the adequate data set, the second being a greedy algorithm building a data set from successive instances that tries to adapt to the current situation. 

\subsection{Batch approach}

In order to allow the user to compare the general behavior of the models, we can select typical scenarios, i.e., predictions on subsets of instances, which can provide information about the decision-maker’s preferences. To this end, we conflate a model $h \in \mathcal{H}$ with its confusion matrix denoted $C(h)$ regarding the instances in some hypothetical ${\mathcal{D'}}$.
Let us consider a set-valued classification problem with $K$ classes. The confusion matrix $C(h)$ of a set-valued classifier can be represented as follows:
{\small \begin{equation}
\label{eq:cm}
    C(h)=\begin{pmatrix} 
    n^h(\{y_1\},y_1) & \ldots & n^h(\{y_1\},y_K)\\
    n^h(\{y_2\},y_1) & \ldots & n^h(\{y_2\},y_K)\\
    n^h(\{y_1,y_2\},y_1) & \ldots & n^h(\{y_1,y_2\},y_K)\\
    \ldots \\ 
    n^h(\mathcal{Y},y_1) & \ldots & n^h(\mathcal{Y},y_K)\\
    \end{pmatrix}
\end{equation}}

where for $A \subseteq \mathcal{Y}$ and $y \in \mathcal{Y}$, $n^h(A,y)$ represents the portion of instances in ${\mathcal{D'}}$ that are predicted $A$ by $h$ while the true class is $y$. But in the case of a utility function as in Equation~\eqref{eq:param_utility}, $C(h)$ can be reduced to the $K \times 2$ following matrix representation: 
\begin{equation}
\label{eq:cm}
    C(h)=\begin{pmatrix}
      R & W \\
     n^h_1 & n^h_{K+1}\\
     n^h_2 & n^h_{K+2}\\
     \ldots & \ldots \\ 
    | n^h_K & n^h_{2K}\\
    \end{pmatrix}
\end{equation}
where $n^h_i$ are percentages of realizations and sum up to $1$, when $i<K+1$, $n^h_i$ represents the number of predictions of $h$ containing $i$ elements including the correct one, and when $i>K$, $n^h_i$ represents a prediction containing $i-K$ elements excluding the correct one. Note that $n^h_{2K}=0$ all the time, as we assume the true class to be in $\mathcal{Y}$.


Suppose we have narrowed down the value of the parameter to $\theta^* \in [\beta, \gamma]$ and consider the utility given by Equation~\eqref{eq:param_utility} together with the inequality from Equation~\eqref{eq:pref_values}. It follows that, depending on the user’s preference between $h$ and $h'$, the following gap will be positive or negative: 
\begin{align} \label{eq:perf_difference}
\Delta(h,h')& := \sum_{x \in \mathcal{D'}} u_\theta(h(x),y) - \sum_{x \in \mathcal{D'}} u_\theta(h'(x),y) \\ &=\sum_{i=1}^K n^h_{i} \left(\frac{\theta(i-1) + 1}{i^2}\right)  - \sum_{i=1}^K n^{h'}_{i} \left(\frac{\theta(i-1)+1}{i^2}\right)\nonumber \\
&= \sum_{i=1}^K (n^h_{i} - n^{h'}_{i}) \left(\frac{\theta(i-1)+1}{i^2}\right) \nonumber
\end{align}
Remark that the above equality only depends on the fractions of samples for which the true class is within the prediction, yet it is to be reminded that increasing $n^{h'}_{i}$ for $i \leq K$ while not increasing any $n^h_{i}$ for $i \leq K$ may require to choose a pair of samples where one is right, the other wrong. This raises the question of whether we should associate non-null utilities to incorrect predictions, or equivalently non-zero loss to correct ones? Discussing such a question in details is however out of the scope of the current study. Coming back to our considered setting, we have 
\begin{equation}\label{eq:delta_null}\Delta(h,h')=0 \Leftrightarrow \theta=\frac{\sum \limits_{i=1}^K (n^{h}_{i} - n^{h'}_{i}) \cdot \frac{1}{i^2}}{ \sum \limits_ {i=1}^K (n^{h}_{i} - n^{h'}_{i}) \cdot \frac{(i-1)}{i^2}},\end{equation}
and clearly, a statement such that $h \succ^\mathcal{D'} h'$ ($h \prec^\mathcal{D'} h'$) would amount to say that $\Delta(h,h')>0$ ($\Delta(h,h')<0$), or $$\theta > \frac{\sum \limits_{i=1}^K (n^{h'}_{i} - n^{h}_{i}) \cdot \frac{1}{i^2}}{ \sum \limits_{i=1}^K (n^{h}_{i} - n^{h'}_{i}) \cdot \frac{(i-1)}{i^2}} \quad (\theta < \frac{\sum \limits_{i=1}^K (n^{h'}_{i} - n^{h}_{i}) \cdot \frac{1}{i^2}}{ \sum \limits_{i=1}^K (n^{h}_{i} - n^{h'}_{i}) \cdot \frac{(i-1)}{i^2}}),$$
thus effectively providing us with an information about the user behaviour towards imprecision and set-valued predictions. Hence, if at some step of the process we know $\theta \in [\beta,\gamma]$, reducing by half this interval (whatever the reply of the user) amounts to finding a data set $\mathcal{D'}_\delta$ with $\delta=(\beta+\gamma)/2$ such that 
$$\frac{\sum \limits_{i=1}^K (n^{h'}_{i} - n^{h}_{i})\cdot \frac{1}{i^2}}{ \sum \limits_{i=1}^K (n^{h}_{i} - n^{h'}_{i}) \cdot \frac{(i-1)}{i^2}}=\delta$$
and asks the user to compare the two results. 

For instance, starting from $ \theta \in [1,3]$, $K=2$ and the two set-valued classifiers $h, h' \in \mathcal{H}$:
\begin{equation*}
    C(h)=\begin{pmatrix}
    n^h_1 & n^h_{3}\\
    n^h_2 & 0\\
    \end{pmatrix}
\end{equation*}
we do have,
$$
h \succ^{\mathcal{D}} h' \Rightarrow n^h_1 + n^h_2 \left( \frac{\theta}{4} + \frac{1}{4} \right) > n^{h'}_1 + n^{h'}_2 \left( \frac{\theta}{4} + \frac{1}{4} \right).
$$

The value $\delta = \nicefrac{(\beta + \gamma)}{2}=2$ is obtained for the choice
\begin{equation}\label{eq:solution_proportions}
\frac{n^{h}_1 - n^{h'}_1}{n^{h'}_2 - n^{h}_2} = \frac{3}{4}
\end{equation}
if we include the constraints $n^{h}_1 - n^{h'}_1 > 0$ and $n^{h'}_2 - n^{h}_2 > 0$, that would correspond to $h'$ being a cautious version of $h$, in addition to $n^{h}_i,n^{h'}_i \in [0,1]$ and $\sum_{i=1}^K n^{h}_1 \leq 1$ (as we do not include false predictions in the equations). Of course, this still leads to having an infinity of solutions, even if we consider $h$ to be a precise classifier, that is setting $n^{h}_2=0$. Another natural solution, if $h$ is a precise classifier, is to set $n^{h}_1$ near the actual accuracy of the classifier, so that the hypothetical data set $\mathcal{D'}$ reflects its behaviour. For instance, setting $n^{h}_1=0.8$ and $n^{h}_2=0$, we could pick as solution $n^{h'}_1=0.65$ and $n^{h'}_2=0.20$, so $20\%$ of correct doubleton predictions. 

\begin{remark}\label{rem:abstract_matrix}
In principle, we could present to the user any pair of matrices of the form~\eqref{eq:cm} that solves Equation~\eqref{eq:delta_null} for a given $\delta$ we want to test, as is done for instance in metric elicitation (see Section~\ref{sec:related_works} for details). This has the advantage of letting a lot of freedom to the elicitation process, but we do think that presenting abstract matrices of numbers that could be unrealistic is not very operational. 
\end{remark}

\subsection{Iterative construction procedure}

The batch procedure that consists in fixing quantities $n_i^h$ to reach the desired value $\theta^*$ with a corresponding confusion matrix is interesting, but practically limited, as it provides no guarantee that a sub-sample satisfying these proportions can be found, and as it requires one to pick a solution among an infinity of possible ones. Indeed, the simple Equation~\eqref{eq:solution_proportions} do have an infinity of solutions. Also, in Echo to Remark~\ref{rem:abstract_matrix}, it may be cognitively difficult for a user non-expert in classification (as is often the case in practice) to compare proportions rather than actual instances. 

\begin{algorithm}
\caption{Classifier robustness elicitation - instance procedure}
\label{alg:gen_elicitation_it}
\KwIn{Instances data set $\mathcal{D}$, target value $\delta$, classifiers $h,h'$, precision $\epsilon$}
\KwOut{Set of instances $\mathcal{D}_{el}$ to show to the user}

$\theta \gets \infty$\;
$\mathcal{D}_{el} \gets \emptyset$\;
\While{$|\theta - \delta| \geq \epsilon$}{
Sample $(x,y) \in \mathcal{D}$ \;
Compute proportions $h^{\mathcal{D}_{el} \cup \{(x,y)\} }$ and $h'^{\mathcal{D}_{el} \cup \{(x,y)\}}$ \;
Solve $\Delta(h,h')=0$ from Equation~\eqref{eq:perf_difference}, obtain $\theta^{\mathcal{D}_{el} \cup \{(x,y)\}}$ \;
\If{$|\theta^{\mathcal{D}_{el} \cup \{(x,y)\}} - \delta \leq |\theta - \delta|$ }
{$\mathcal{D}_{el} \gets \mathcal{D}_{el} \cup \{(x,y)\}$ \;
$\mathcal{D} \gets \mathcal{D} \setminus \{(x,y)\}$ \;
$\theta \gets \theta^{\mathcal{D}_{el} \cup \{(x,y)\}}$}
}
\Return{$\mathcal{D}_{el}$}
\end{algorithm}

We therefore propose an iterative algorithm, that only makes the assumption that the test set $\mathcal{D}$ will contain a sufficiently rich amount of situations where $h$ and $h'$ will not provide the same predictions. Algorithm~\ref{alg:gen_elicitation_it} presents a procedure to iteratively build a data set to present to the user. It can be sped up by ensuring that $\mathcal{D}$ only contains instances for which the two classifiers provide different predictions. 

Note that the algorithm will always produce a set of actual examples for which the two classifiers are different. In theory, it does not really matter whether $h,h'$ are good or bad classifiers, as we mainly want to measure the end-user disposition towards imprecision and set-valued instances, but in practice $h,h'$ should be reasonably good yet not perfect classifiers: if they are too bad, the user might not engage in the comparisons or find them irrelevant; if they are too good, there are high chances that we will not find situations where the added imprecision is beneficial.

\section{Related works}
\label{sec:related_works}

While we are not aware of other works that explicitly have the same intent as the one we have, there are other lines of work that are connected to ours.

\paragraph{Metric elicitation} this line of work \cite{hiranandani2019performance,hiranandani2020fair} is close to our approach yet is dedicated for point-prediction classification problems. To simplify the comparison, we discuss only the binary case. 
Metric elicitation approach consider two kinds of utilities to quantify the performance of a classifier $h$ from its confusion matrix $C(h)$ represented by two values: true positives (TP) and true negatives (TN).
They are the linear performance metric 
of whom the accuracy metric is an instance, and
the linear-fractional performance metric
of whom the $F_{\beta}$ measure and the Jaccard similarity coeficient (JAC) are instances.

When considering set-valued classification problems, evaluating with the utility $u_{\alpha}$ defined as in Equations (\ref{eq:param_utility}) and (\ref{eq:discount_utility}):
corresponds to considering a linear performance metric.
Indeed this evaluation corresponds to a linear combination of the coefficients in the confusion matrix in Equation~\ref{eq:cm}:
\begin{align}
\sum_{x \in \mathcal{D}} u_\theta(h(x),y)=n^h_{1} + \left(\frac{\theta + 1}{4}\right) \cdot n^h_{2} 
\end{align} 
In the proposed setting, utility is solely determined by the size of the prediction set. In this case, the difference between metric elicitation and our approach lies mainly in the structural properties of the confusion matrix. However, when the decision maker’s utility is intended to model cautiousness with respect to specific errors that may have catastrophic consequences, the problem becomes more involved. In such situations, it is necessary not only to identify the parameters of the utility function, but also to elicit the utilities associated with different cautious prediction sets \cite{kunitomo2025towards}.
 Furthermore, given the formal similarities, it would be interesting to investigate how the developments within metric elicitation can help to analyse the problem of robustness elicitation.

\paragraph{Credal proper scoring} recent lines of work~\cite{singh2025truthful,frohlich2024scoring} have proposed to extend the notion of proper scoring rules to the imprecise probability setting, a problem that already received quite some attention in the past~\cite{seidenfeld2012forecasting}. Such works are clearly connected to ours, as they explore how evaluation functions can incentivize the classifier (or the forecaster) to provide truthful probability sets, from which can be extracted set-valued predictions. 

The closest work in this line to ours is~\cite{singh2025truthful}, yet it remains mainly on the theoretical side and does not provide a fully operational protocol to elicit the proper scoring rule or decision maker attitude. It is also dedicated to credal sets, while our approach is agnostic to the way set-valued predictions are produced. It would clearly be interesting to close the gap between these more theoretical studies and our operational approach, so that one can nourish the other, and vice-versa.

\begin{figure*}[t]
\centering
\includegraphics[width=0.8\textwidth, trim={0 115pt 0 0},clip]{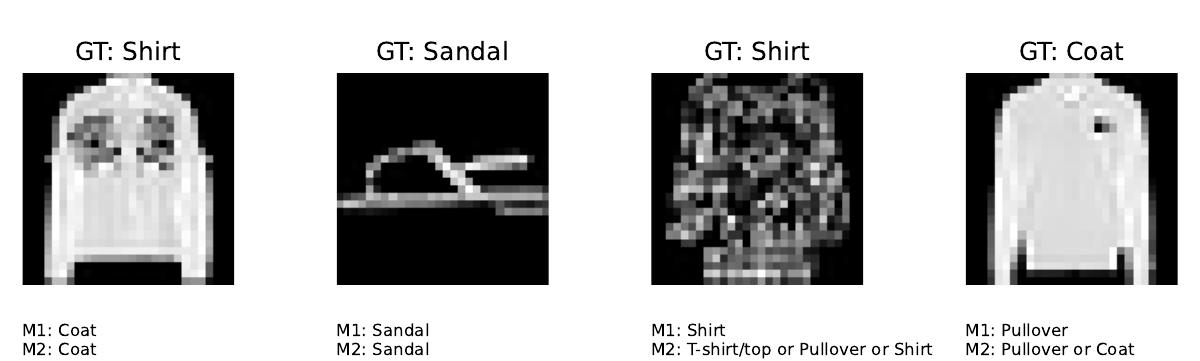}
\caption{Elicitation Examples for the Fashion MNIST dataset for $\alpha^*=1.4$. GT stands for ground-truth of the image. $M1$ is the prediction of the precise classifier, $M2$ of the (possibly) imprecise classifier}
\label{fig:fashion_mnist_elicitation_examples_alpha_1.4}
\end{figure*}

\section{Experimental results}
\label{sec:experiments}

\begin{figure*}
\centering
\includegraphics[width=0.8\linewidth]{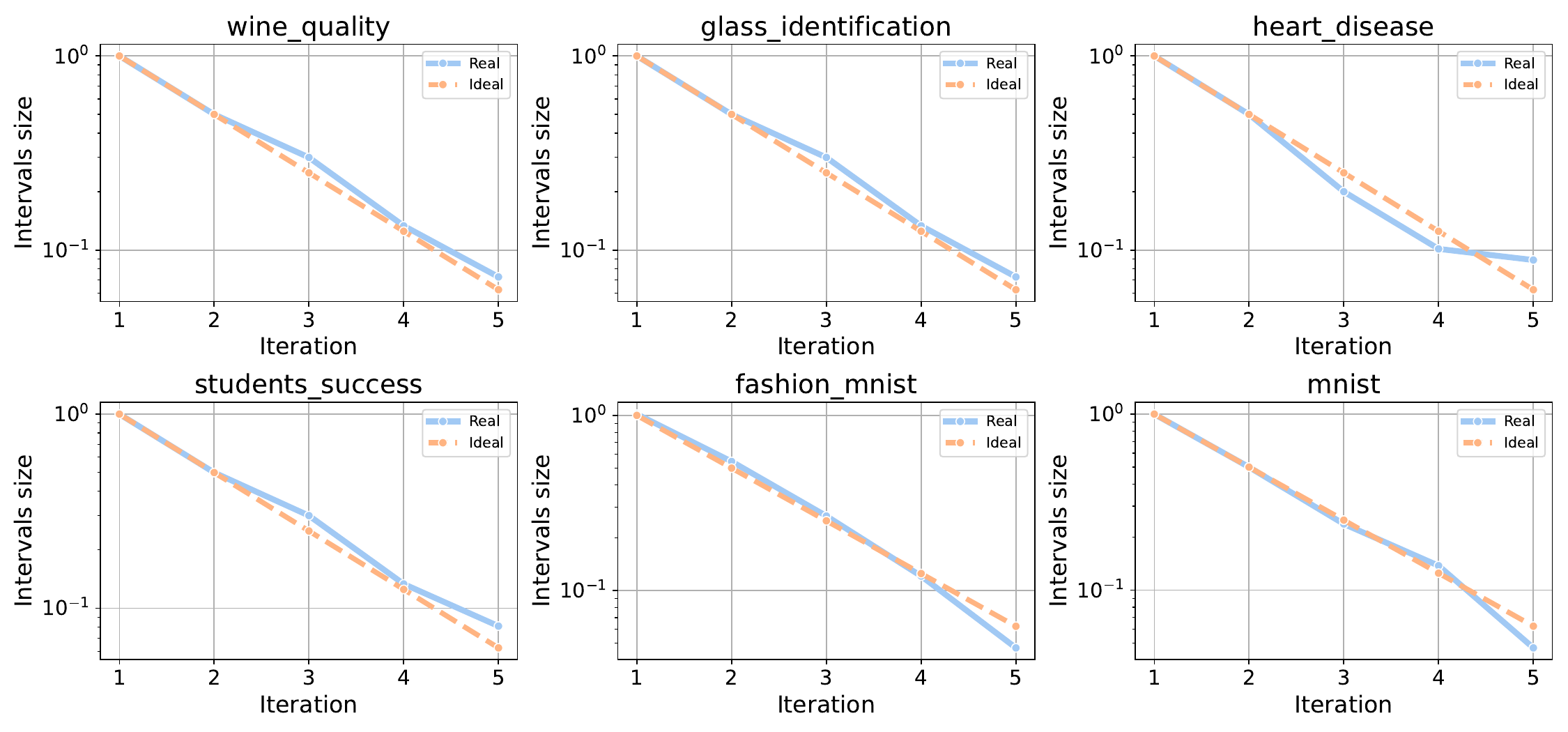}
\caption{Evolution of the interval size for parameter $\alpha^*$ across iterations for six datasets using Algorithm~\ref{alg:gen_elicitation_it}). The y-axis shows the interval size $[\beta, \gamma]$ in logarithmic scale, demonstrating exponential convergence towards the true parameter value. 
The ideal curve (orange) corresponds to halving the interval size at each iteration, while the actual curve (blue) deviates slightly from this ideal behavior due to the tolerance parameter $\epsilon$ in Algorithm~\ref{alg:gen_elicitation_it}.}
\label{fig:all_intervals_size_logscale}
\end{figure*}

The experiments were conducted on six datasets: four tabular datasets from the UCI Machine Learning Repository (Wine Quality, Glass Identification, Heart Disease, and Student Success) \cite{duaUCIRepository2019}, and two image classification datasets (Fashion-MNIST and MNIST) \cite{lecunMNISTHandwrittenDigit1998,xiaoFashionMNISTNovelImage2017}. For the tabular datasets, the data were split into training and test partitions using a 70/30 split. The family of models to compare was generated using XGBoost \cite{chenXGBoostScalableTreeBoosting2016} classifiers configured to produce multiclass probability estimates. These probabilities were then post-hoc calibrated with logistic regression, and set-valued predictions were obtained by maximizing the utility function in $\eqref{eq:param_utility}$ over different values of $\theta$. For the image datasets (Fashion-MNIST and MNIST), we trained neural network classifiers and generated set-valued predictions using the learning nondeterministic classifiers (NDC) method \cite{delcozLearningNondeterministicClassifiers2009}, varying the desired trade-off between Recall and Precision. In all cases, the same preprocessing, model selection, and evaluation protocol were used across datasets to make the results directly comparable.

\begin{table}[h]
\centering

  \begin{tabular}{lrrrrr}
\toprule
Iteration& \hspace{1em}\#1 & \hspace{1em}\#2 & \hspace{1em}\#3 & \hspace{1em}\#4 & \hspace{1em}\#5 \\
\midrule
wine quality & 8 & 8 & 7 & 7 & 11 \\
glass identification & 8 & 8 & 7 & 6 & 11 \\
heart disease & 6 & 9 & 9 & 8 & 10 \\
students success & 5 & 8 & 8 & 7 & 10 \\
fashion mnist & 7 & 9 & 12 & 4 & 6 \\
mnist & 4 & 5 & 8 & 10 & 5 \\
\bottomrule
\end{tabular}

\caption{Number of samples retrieved by Algorithm~\ref{alg:gen_elicitation_it}) to reach the target value $\alpha^*$. 
}
\label{tab:n_samples}
\end{table}

\begin{figure}
\centering
\includegraphics[width=\linewidth, trim={0 115pt 450pt 0},clip]{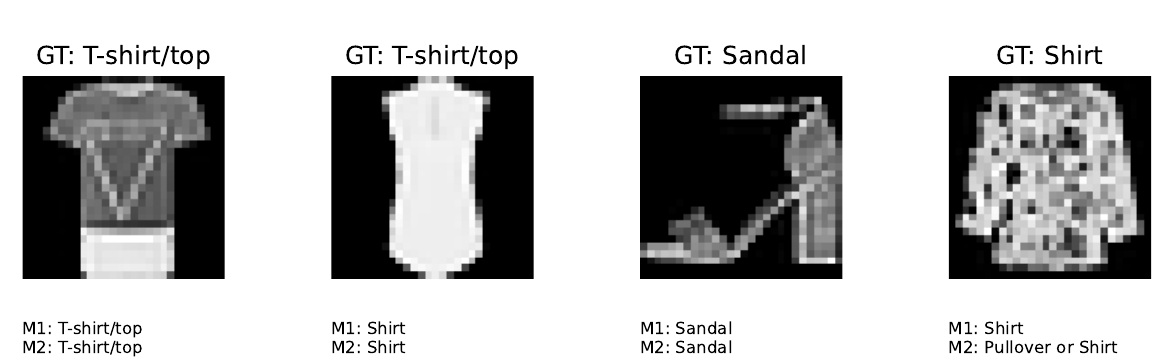}
\includegraphics[width=0.71\linewidth, trim={600pt 115pt 0 0},clip]{fashion_mnist_elicitation_examples_alpha_2.2.pdf}
\caption{Elicitation Examples for the Fashion MNIST dataset for $\alpha^*=2.2$. 
}
\label{fig:fashion_mnist_elicitation_examples_alpha_2.2}
\end{figure}

We choose a user value of $\theta^*=1.4$, and Algorithm \ref{alg:gen_elicitation_it} was configured with a tolerance of $\epsilon=0.05$. Similarly, we stopped the iterative procedure whenever the interval for $\theta$ was smaller than $0.05$.  Since the algorithm is sensitive to random initialization and may require multiple attempts to achieve convergence, we implemented a maximum of 50 trials per iteration. With this setting, the algorithm successfully converged for all datasets. The code is implemented in Python 3.13.5 and is available at 
\url{https://anonymous.4open.science/r/elicitation_imprecision_utilities-11A2}.
Figure \ref{fig:all_intervals_size_logscale} illustrates the convergence of the interval size for $\theta^*$ across iterations, demonstrating exponential convergence towards the true parameter value. Table \ref{tab:n_samples} details the number of samples required at each iteration to reach the target $\theta^*$ for each dataset.

Figures \ref{fig:fashion_mnist_elicitation_examples_alpha_1.4} and \ref{fig:fashion_mnist_elicitation_examples_alpha_2.2} showcase elicitation examples for Fashion-MNIST at $\theta^*=1.4$ and $\theta^*=2.2$, respectively, highlighting the trade-offs between informativeness and accuracy in the predictions presented to users. In an actual elicitation session, the user would be presented with such examples and asked to express their preference between the two models' predictions, thereby providing information about their tolerance for imprecision. The results demonstrate the effectiveness of the iterative elicitation procedure in converging to the target parameter value. The number of examples required varies across datasets but consistently lies within a manageable range, indicating the practicality of the approach for real-world applications. Indeed, on the tested examples, they do not go beyond 12, and it seems reasonable to assume that a domain expert or a user would be able and willing to compare a dozen (at most) items in order to identify an optimal behaviour, and therefore obtain better results for him/her.


\section{Conclusion and perspectives}

In this paper, we have introduced various ways to elicit the robustness level a user is willing to accept when this robustness is translated by set-valued classification, hopefully reflecting our uncertainty about a given prediction. We have illustrated and derived our approach on a commonly used utility function of the form given by Equation~\eqref{eq:param_utility}, for which $\Delta(h,h')=0$ can be solved efficiently. Our experiments also show that an incremental elicitation procedures showing to the user specific instances would require a reasonable number of examples to be presented to the decision maker. 

The most obvious follow-up to the current paper would be to proceed to empirical experiments involving actual users, especially as Table~\ref{tab:n_samples} indicates that it would put an acceptable burden on such users. However, our discussion also opens the door to other interesting questions, such as whether set-valued utility functions extending 0/1 utility should always be null when the produced set does not include the ground-truth class.

\bibliography{Bibli}






\end{document}